\documentclass[letterpaper]{article}
\usepackage{aaai2027}
\usepackage[hyphens]{url}
\usepackage{graphicx}
\usepackage{natbib}
\usepackage{caption}
\usepackage{tcolorbox}
\usepackage{algorithm}
\usepackage{algorithmic}
\usepackage{booktabs}
\usepackage{amsmath,amssymb,amsthm}
\usepackage{multirow}

\newtheorem{proposition}{Proposition}

\newtheorem{assumption}{Assumption}

\title{Matching Supervision to the Student's Learning Capacity: \\
A Unified Framework for On-Policy Self-Distillation}
\author{
    Yongkang Yang\textsuperscript{1}\thanks{\raggedright
    These authors contributed equally.},
    Zhezheng Hao\textsuperscript{2}\footnotemark[1],
    Hong Zhang\textsuperscript{3},
    Yi Liu\textsuperscript{1},
    Xiankun Lin\textsuperscript{1},\\
    Wence Ji\textsuperscript{1},
    Fanjunduo Wei\textsuperscript{2},
    Jiarui Yu\textsuperscript{1},
    Qiang Lin\textsuperscript{1},
    Xiaoyun Liang\textsuperscript{1}\thanks{\raggedright
    Corresponding authors.},
    Hande Dong\textsuperscript{1}\footnotemark[2]\\[5pt]
    \normalsize Emails: \url{eoncanyang@tencent.com},
\url{donghd66@gmail.com}
}

\affiliations{
\Large
    \textsuperscript{1}Tencent, Workbuddy Team \,\,
    \textsuperscript{2}Zhejiang University \,\,
    \textsuperscript{3}Independent Researcher
}

\begin{document}

\maketitle

\begin{abstract}
On-policy self-distillation (OPSD) improves the reasoning abilities of LLMs by internalizing privileged context into model parameters through self-distillation.
Two recent research lines promote vanilla OPSD by choosing \emph{which tokens to learn from} and by controlling \emph{how much privileged information} the teacher receives, respectively.
However, we show that each line optimizes one variable while holding the other fixed, which leads to a suboptimal solution.
We argue that the two variables are coupled through the student's learning capacity: the privileged information sets the per-token divergence the teacher prescribes, while token weighting selects which of these the student must absorb.
We formalize the two lines of work into a unified optimization framework, which maximizes the aggregate teacher--student divergence, subject to a budget on the aggregate learning difficulty the student can absorb.
Under this modelling, we propose \textbf{U}nified on-policy \textbf{S}elf-\textbf{D}istillation, called \textbf{USD}, a lightweight online algorithm to solve the Lagrangian.
USD reveals that a single dual variable governs both decisions: at one price for learning difficulty, it simultaneously sets the token-selection threshold and the direction of privileged-information adjustment, keeping supervision matched to the student's evolving capacity.
Through extensive experiments, USD consistently demonstrates superior performance over OPSD and token- and PI-side baselines across various model scales on various reasoning benchmarks\footnote{Code is available at https://github.com/lauvlalala/USD}.
\end{abstract}



\section{Introduction}

Post-training has become a key technique for improving the reasoning capability of large language models~\citep{openai5gpt,gemini3, anthropic2025claude45}.
Among the emerging paradigms, on-policy self-distillation (OPSD)~\citep{zhao2026self} is particularly appealing because it requires neither curated reasoning trajectories (like SFT~\citep{yuan2023scaling,yu2023metamath}), an interactive environment (like RL~\citep{cobbe2021training,ouyang2022training,yu2026dapo}), nor an externally stronger teacher (like distillation~\citep{hinton2015distilling,kim2016sequence,gu2024minillm}); instead, it trains a model with only a piece of \emph{privileged information (PI)}---additional context that is available at training time but not at deployment, such as verified reference solutions, environment feedback, ground-truth reasoning chains, or logged experience.
Within a single model, OPSD assigns two roles: a student generates its own rollouts from the ordinary problem prompt, and a \emph{privileged teacher}---the same base model conditioned on the PI---scores each generated token to produce fine-grained supervision.
Since the student learns from the states it actually visits and the teacher shares its parameters, this signal is both distribution-consistent and free of the capacity gap that plagues cross-model distillation~\citep{agarwal2024policy,furlanello2018born}. OPSD brings substantial gains on competition math reasoning, code and agentic benchmarks~\citep{zhao2026self,tan2026paint,zhong2026sod}, exceeding reinforcement learning with a much smaller compute budget.

Despite its success, recent work reveals that the two default settings of vanilla OPSD can limit its performance: uniform token weighting and full-context PI exposure.
At the \emph{token level}, uniform weighting is inefficient: the gradient signal is concentrated in a small fraction of positions, with the top 20\% carrying roughly 80\% of the mass~\citep{wang2026not,xiao2026finding}. High-divergence positions further compound the issue, since many of them reflect \emph{incompatible} rather than learnable disagreement---teacher mass that lies outside the student's support and thus yields no useful gradient.
At the \emph{PI level}, conditioning the teacher on a piece of over-privileged context pushes its distribution too far from the student's, where the teacher--student divergence cannot be closed by a single gradient step~\citep{han2026adaptive,tan2026paint}. When the privileged context also contains the target answer explicitly, full exposure leaks it into the supervision signal. Two lines of work address these limitations respectively: \emph{token selection} methods measure each token position and re-weight the loss accordingly~\citep{wang2026not,liu2026teacher,xu2026tip,li2026filter,xie2026position}, while \emph{PI control} methods adapt how much of the privileged context is revealed at training time~\citep{tan2026paint,han2026adaptive,zhang2026beyond,shen2026purified}.

\begin{figure}[t]
\centering
\includegraphics[width=\columnwidth]{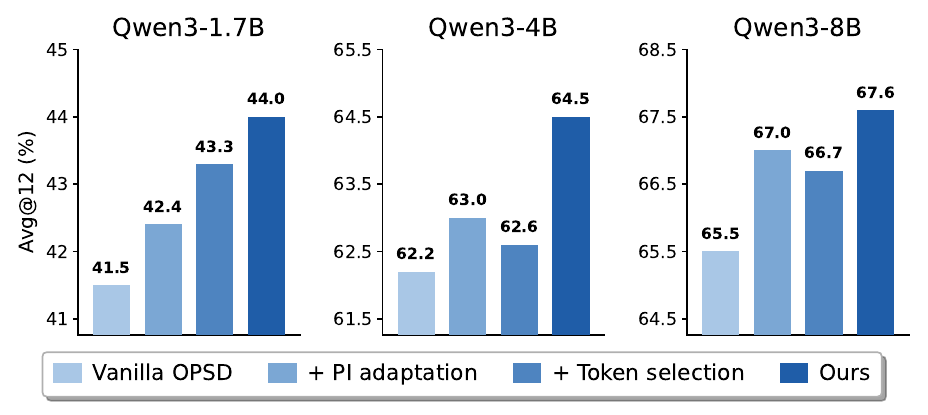}
\caption{Performance comparison (Avg@12, \%) on math reasoning tasks across three Qwen3 scales. Each single-axis variant improves over vanilla OPSD, and the joint variant is uniformly best across scales.}
\label{fig:teaser}
\end{figure}

We argue that the two variables are not independent: PI shapes the per-token teacher--student divergence, while token weighting selects which of these divergences enter the loss. Each variable's optimal setting depends on the other. Lowering PI reduces per-token learning difficulty across the board, admitting positions the student previously could not absorb; conversely, filtering the hardest-to-absorb tokens lowers the aggregate difficulty, admitting a PI level that would otherwise have overwhelmed the student.
\emph{Optimizing either variable under a fixed setting of the other reaches only a conditional optimum, not the joint one.}

Motivated by the coupling between token selection and PI, we formulate OPSD as an optimization problem of matching supervision to the student's current learning capacity.
Each correction the teacher prescribes both opens a divergence, the amount the student stands to learn, and incurs a learning difficulty, the distance between the correction and what the student currently believes. A student can productively absorb only a limited amount of such difficulty at a given stage of training. We therefore jointly choose the token weights and the amount of privileged information to maximize the total divergence while keeping the aggregate learning difficulty within the student's capacity. We refer to the best divergence attainable at each capacity level as the \emph{Learnability Frontier}.
Token-selection and PI-control methods optimize this trade-off along only one dimension, whereas our formulation searches for the capacity-matched point by adapting both dimensions together.

We solve this joint optimization problem by introducing a single capacity budget as a Lagrangian constraint.
The resulting dual variable coordinates both decisions in either direction: when the aggregate difficulty exceeds the budget, it tightens token selection and lowers PI; when the budget has slack, it relaxes selection and raises PI so the student is not under-supervised.
In this way, we propose \textbf{U}nified on-policy \textbf{S}elf-\textbf{D}istillation, called \textbf{USD}, a lightweight online algorithm that solves this Lagrangian during training.
It measures the aggregate difficulty the student is absorbing and uses it to 
update two scalars, the dual variable $\lambda$ and the PI strength $\beta$, from which the token weights follow in closed form.
Empirically, USD consistently outperforms vanilla OPSD across model scales and datasets. Ablations show that each axis helps on its own and that the two are complementary (Figure~\ref{fig:teaser}). Training dynamics further show that the two decisions co-evolve to keep the aggregate learning difficulty near the target budget.

Overall, our contributions can be summarized as follows:
\begin{enumerate}
\item We identify the coupling between token-level supervision and teacher-level privileged information in OPSD: optimizing either one under a fixed setting of the other yields only a suboptimal solution (\S\ref{sec:supervision}).
\item We formulate capacity-matched OPSD as a joint optimization of token weights and PI level, and introduce the \emph{Learnability Frontier} to characterize the best divergence attainable under a bounded aggregate difficulty (\S\ref{sec:joint}).
\item We propose USD, a lightweight online algorithm that coordinates soft token weighting and PI adaptation through a single dual variable, raising the mean accuracy from 56.4 to 58.7 over vanilla OPSD on mathematical reasoning benchmarks (\S\ref{sec:algorithm} and \S\ref{sec:experiments}).
\end{enumerate}

\section{Related Work}
\label{sec:related}

\paragraph{On-Policy Self-Distillation.}
On-policy distillation (OPD) trains a student on its own rollouts under dense supervision from a teacher~\citep{agarwal2024policy,hou2026uni}, and has become a widely adopted post-training paradigm for LLM reasoning~\citep{song2026survey,wei2022chain}. A prominent line within OPD is \emph{self-distillation}, which repurposes a single base model as both teacher and student under different context conditions~\citep{shenfeld2026self,hubotter2026reinforcement,zhao2026self,penaloza2026privileged,yang2026self,stein2026gates,kim2026opsd, ke2026respecting,li2026demopsd}. However, recent work identifies a persistent mismatch between the teacher's supervision and the student's current learning capacity~\citep{kim2026does,wang2026denser,li2026rethinking,zhang2026full}, which our joint formulation is designed to address.

\paragraph{Token Selection for On-Policy Distillation.}
Vanilla OPD treats every token equally, but recent studies observe two limitations of this uniform treatment: task-bearing supervision is concentrated in a small fraction of positions~\citep{xu2026tip,wang2026not}, while the remaining bulk of the loss is spent on stylistic or format-driven tokens that can even induce privilege-induced style drift~\citep{pan2026rlcsd,yang2026ogls}. To sharpen the signal, a growing body of work reweights or gates the token-level loss according to signals such as student entropy, teacher--student divergence, sequence position, local teachability, and trust-region support~\citep{xu2026tip,wang2026not,liu2026teacher,xie2026position,liu2026prefix,li2026filter,jin2026entropy,hao2026rethinking,xing2026trust,yang2026prune,zhang2026tailoring,ziheng2026less,wang2026trace}. Across these proposals, PI is treated as a fixed input and token weights are optimized conditional on the teacher signal; our framework instead adapts token weighting and PI strength jointly under a shared learning-capacity budget.

\paragraph{Privileged Information Control in On-Policy Distillation.}
A parallel line of work observes that full privileged exposure creates too large a teacher--student divergence for the student to absorb in one step and can leak answer-dependent shortcuts into the supervision signal, and reduces the effective PI reach through mechanisms such as adaptive solution masking, learned exposure schedules, anchored residual application, and reference-signal subtraction~\citep{tan2026paint,han2026adaptive,zhang2026beyond,shen2026purified}. These OPD-specific instantiations echo classical exposure-control ideas from scheduled sampling~\citep{bengio2015scheduled} and DAgger~\citep{ross2011reduction}. 
Most of these schedules adjust PI without measuring its effect: the exposure level follows a predetermined rule or a heuristic signal, rather than the difficulty the student is currently absorbing.
Across these proposals, PI is adapted while token-level supervision is applied uniformly across positions; our framework instead adapts token weighting and PI strength jointly under a shared learning-capacity budget.

\section{Methodology}
\label{sec:method}

\begin{figure}[t]
\centering
\includegraphics[width=0.85\columnwidth]{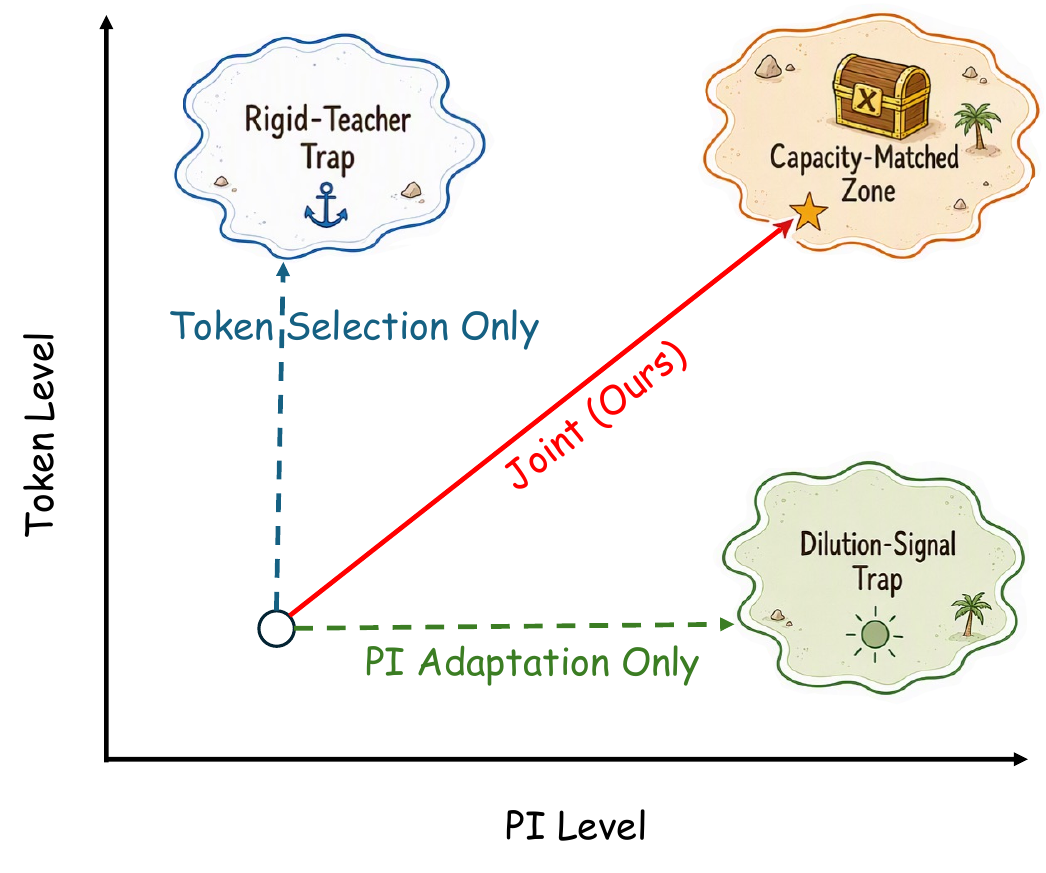}
\caption{Two design axes of OPSD. Existing methods adapt either token weighting or PI exposure alone, whereas USD jointly adapts both axes.}
\label{fig:axes}
\vspace{-1em}
\end{figure}

\subsection{Two Design Axes of OPSD}
\label{sec:supervision}

In on-policy self-distillation, the student $p_S$ generates a rollout $\hat{y}$ conditioned on the problem $x$, while the teacher $p_T^\beta$---the same base model conditioned on privileged information---provides token-level supervision. The scalar $\beta \in [0,1]$ controls the PI exposure, from none at $\beta=0$ to full at $\beta=1$. We further introduce per-token \emph{weights} $w_t \in [0,1]$, giving the general weighted objective
\begin{equation}
J(\mathbf{w}, \beta) = \frac{1}{\sum_t w_t}\sum_{t=1}^{|\hat{y}|} w_t \cdot D_{\text{KL}}\!\left(p_T^\beta(\cdot \mid x, \hat{y}_{<t}) \,\|\, p_S(\cdot \mid x, \hat{y}_{<t})\right).
\label{eq:opsd}
\end{equation}
Vanilla OPSD~\citep{zhao2026self} corresponds to $\beta = 1$ and a uniform weights of $w_t \equiv 1$.

These two axes control complementary aspects of privileged supervision.
As illustrated in Figure~\ref{fig:axes}, they define a two-dimensional design space for OPSD.
Existing methods typically optimize only one axis at a time:
token-selection approaches move vertically by reallocating supervision under a fixed teacher,
whereas PI-adaptation approaches move horizontally by adjusting the teacher while uniformly supervising all tokens.
Our method instead jointly adapts both axes, allowing supervision to better match the student's current learning capacity.
The PI-strength axis, governed by $\beta$, changes the teacher distribution itself: increasing $\beta$ exposes more privileged information and typically produces a sharper, more corrective teacher. The token-weighting axis, governed by $\mathbf{w}$, instead controls where and to what extent this signal is applied along the student rollout. Thus, $\beta$ determines \emph{what supervision is generated}, whereas $\mathbf{w}$ determines \emph{which parts of that supervision are transmitted}.
\paragraph{Interpretation.}
The two control axes play fundamentally different roles. Adjusting the PI strength changes the teacher itself, affecting every token simultaneously by reshaping the privileged distribution. Adjusting token weights, in contrast, does not alter the teacher, but selectively allocates the resulting supervision across the student rollout. One axis therefore controls the quality of supervision, while the other controls its allocation. Since both ultimately determine how much useful supervision the student receives, they should be optimized jointly rather than independently.

To reason about how the axes should be set, we characterize the per-token teacher signal by two quantities. The \emph{divergence}
\[
g_t(\beta) \;=\; D_{\text{KL}}\!\left(p_T^\beta \,\|\, p_S\right)_t
\]
measures how far the teacher's distribution lies from the student's at position $t$---the potential update magnitude at that token. Let $a_t^\star = \arg\max_a p_T^\beta(a \mid x, \hat{y}_{<t})$ denote the teacher's top prediction. The \emph{learning difficulty}
\[
d_t(\beta) \;=\; \left(\log \frac{p_T^\beta(a_t^\star)}{p_S(a_t^\star)}\right)_{\!+}
\]
measures how far the teacher's top prediction lies above the student's current belief---the effort required to absorb the correction; we take the positive part since a token the student already prefers ($p_S(a_t^\star) \ge p_T^\beta(a_t^\star)$) poses no learning burden. Intuitively, $d_t$ is the student's surprise at the teacher's preferred token: when the student already assigns it high probability the correction is a small step, whereas when the teacher's choice lies far outside the student's expectation $d_t$ is large and a single update cannot close the gap. Both quantities depend on $\beta$: strengthening PI sharpens the teacher, typically raising both the correction it prescribes and the effort to absorb it. A useful signal therefore requires balancing the two: the teacher strong enough to offer substantial correction, yet not so strong that its recommendations exceed the student's capacity to absorb.

The two axes, moreover, are coupled through these quantities. Reducing $\beta$ shrinks per-token difficulty across the board and makes previously unabsorbable tokens learnable, so the set of tokens worth including shifts; conversely, filtering high-difficulty tokens lowers the aggregate load and admits a higher $\beta$. Optimizing either axis under a fixed setting of the other therefore yields only a conditional solution.

\subsection{The Unified Optimization Framework}
\label{sec:joint}

\begin{figure}[t]
\centering
\includegraphics[width=\columnwidth]{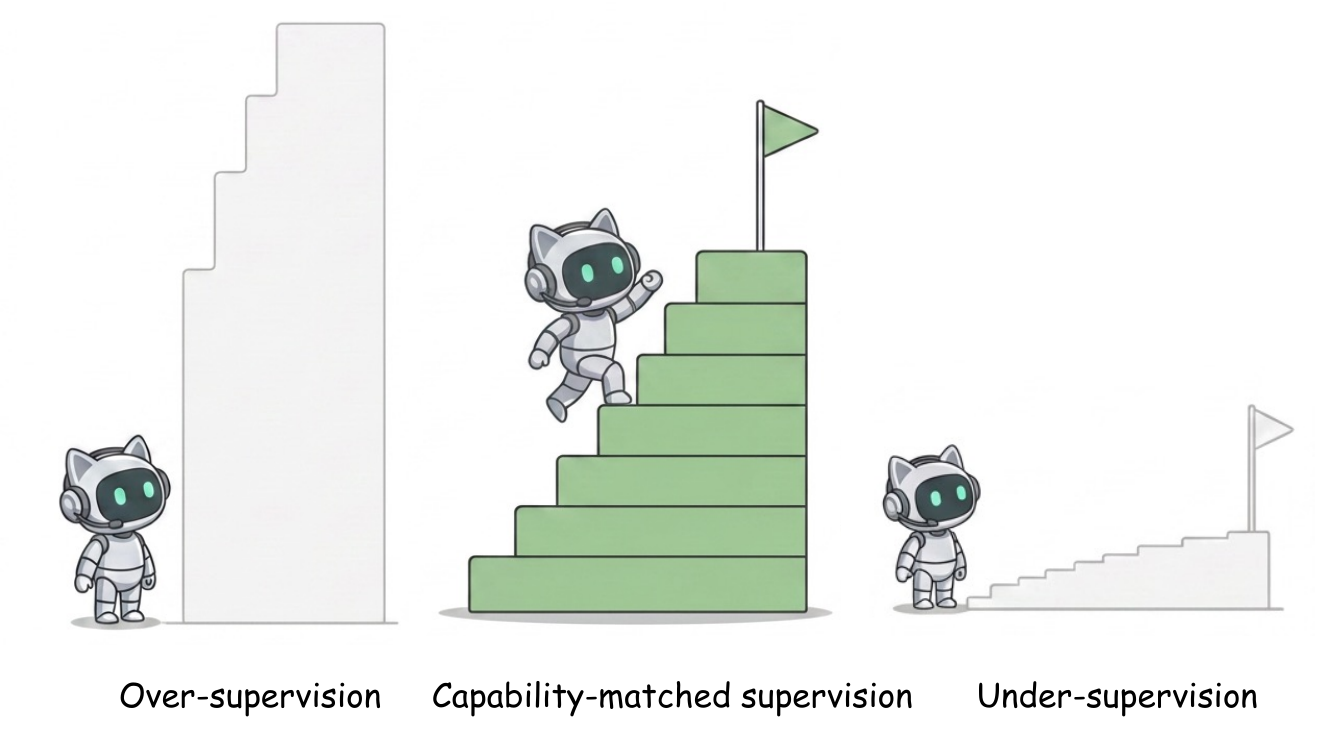}
\caption{Over-, matched, and under-supervision.}
\label{fig:capacity_matching}
\vspace{-1em}
\end{figure}

As illustrated in Figure~\ref{fig:capacity_matching}, effective supervision should match the student's current learning capacity. When supervision exceeds the student's capacity, the resulting corrections are too difficult to absorb; when it falls far below that capacity, the learning signal is too weak to support substantial improvement.
We formalize the joint choice of the two axes as a constrained optimization: maximize the mean weighted divergence the supervision delivers, subject to a budget on the mean weighted difficulty the student must absorb,
\begin{equation}
\max_{\mathbf{w} \in [0,1]^T,\; \beta \in [0,1]} \frac{1}{N}\sum_t w_t \, g_t(\beta) \quad \text{s.t.} \quad \frac{1}{N}\sum_t w_t \, d_t(\beta) \leq \varepsilon,
\label{eq:constrained}
\end{equation}
where $\varepsilon > 0$ is the student's per-batch \emph{capacity budget}---the mean weighted learning difficulty the student can absorb at the current stage. Selecting more tokens and strengthening PI both raise this quantity and consume the same budget, so the two decisions compete for the same resource.

Sweeping the budget traces out the \emph{Learnability Frontier} $\mathcal{F}(\varepsilon)$, the maximum mean weighted divergence attainable at each capacity level:
\begin{equation}
\mathcal{F}(\varepsilon) \;=\; \max_{\mathbf{w}, \beta}\left\{\,\frac{1}{N}\sum_t w_t \, g_t(\beta) \;:\; \frac{1}{N}\sum_t w_t \, d_t(\beta) \leq \varepsilon\,\right\}.
\label{eq:frontier}
\end{equation}
The frontier is the Pareto boundary of the divergence--difficulty trade-off. Token-selection methods (fixing $\beta$) and PI-control methods (fixing $\mathbf{w}$) each restrict the search to a one-dimensional slice of this landscape and reach the frontier only by coincidence; our method optimizes both axes jointly to move along it.

\subsection{Theoretical Analysis}
\label{sec:structure}
The unified optimization problem admits a simple interpretation. Since both token weighting and PI strength consume the same learning-capacity budget, they compete for a shared resource rather than acting independently. The Lagrangian therefore introduces a single dual variable that quantifies the marginal value of learning difficulty and naturally couples the two decisions.

Solving Eq.~\eqref{eq:constrained} reveals that the two supervision axes are governed by this single quantity---the marginal price of learning difficulty. Applying the Lagrangian~\citep{boyd2004convex} to Eq.~\eqref{eq:constrained} with dual variable $\lambda \ge 0$ yields
\begin{equation}
\mathcal{L}(\mathbf{w}, \beta, \lambda) \;=\; \frac{1}{N}\sum_t w_t \underbrace{[g_t(\beta) - \lambda\, d_t(\beta)]}_{v_t(\lambda, \beta)} + \lambda\varepsilon,
\label{eq:lagrangian}
\end{equation}
where each token contributes a \emph{net value} $v_t = g_t - \lambda d_t$: its divergence minus $\lambda$-priced difficulty. Since $\mathcal{L}$ is linear in $w_t$ with box constraints, the optimal weights collapse to a hard threshold,
\begin{equation}
w_t^\star \;=\; \mathbf{1}\!\left[g_t(\beta) > \lambda \, d_t(\beta)\right].
\label{eq:hard_weights}
\end{equation}
A token is thus included exactly when its divergence exceeds its $\lambda$-priced difficulty, i.e.\ when its net value $v_t$ is positive; token selection reduces to a per-token $0/1$ decision by the sign of $v_t$, with $\lambda$ setting the price at which divergence is traded against difficulty.

The dual variable also fixes the optimal PI strength. When the budget binds ($\lambda > 0$) at an interior optimum $\beta^\star \in (0,1)$, the stationarity condition $\partial\mathcal{L}/\partial\beta = 0$ over the selected tokens gives
\begin{equation}
\lambda \;=\; \frac{\sum_{t:\,w_t^\star=1} \partial g_t/\partial \beta}{\sum_{t:\,w_t^\star=1} \partial d_t/\partial \beta}.
\label{eq:beta_stationarity}
\end{equation}
This condition reads as a marginal balance: at the optimal PI strength, the extra divergence that a small increase in $\beta$ contributes, summed over the selected tokens, equals $\lambda$ times the extra difficulty it incurs. The same $\lambda$ that prices difficulty in the token decision reappears here as the exchange rate between divergence gained and difficulty added, and $\beta$ is raised until one more unit of difficulty no longer buys $\lambda$ units of divergence.

\begin{tcolorbox}[
  colback=gray!5,
  colframe=gray!55!black,
  title=\textbf{Takeaway},
  fonttitle=\bfseries\small,
  boxsep=3pt,
  left=6pt, right=6pt, top=4pt, bottom=4pt,
  arc=2pt,
]
\small
The two supervision axes---token weighting and PI strength---are coupled through a single \emph{capacity budget}: both decisions consume the same aggregate learning difficulty. Solving the constrained problem shows that one dual variable $\lambda$, the marginal price of difficulty, simultaneously sets the token-selection threshold and the direction of PI adjustment---not by design, but as a structural consequence of the shared budget.
\end{tcolorbox}

Condition~\ref{eq:beta_stationarity} cannot be evaluated directly, because $\beta$ acts through non-differentiable operations on the teacher's context. A practical surrogate is to move $\beta$ along the primal residual $\varepsilon - \mathrm{load}(\beta)$, where $\mathrm{load}(\beta) = \frac1N\sum_t w_t d_t(\beta)$ is the current mean weighted difficulty. Under a single monotonicity assumption, this rule reaches the budget-binding optimum without ever forming a gradient in $\beta$.

\begin{assumption}[Monotonicity in $\beta$]
\label{asm:beta}
Over the selected tokens, both the mean weighted difficulty $\mathrm{load}(\beta) = \frac{1}{N}\sum_t w_t\, d_t(\beta)$ and the mean weighted divergence $\frac{1}{N}\sum_t w_t\, g_t(\beta)$ are strictly increasing in $\beta$.
\end{assumption}

Both parts express the same mechanism: revealing more privileged information sharpens the teacher, raising both the correction it prescribes and the difficulty of absorbing it. Figure~\ref{fig:monotonicity} verifies both monotonicities directly on Qwen3-1.7B training rollouts, with diminishing marginal returns as the reference solution is progressively revealed.

\begin{figure}[t]
\centering
\includegraphics[width=\columnwidth]{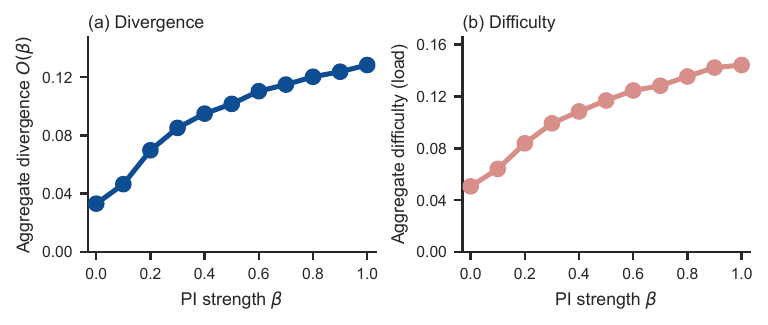}
\caption{Empirical verification of Assumption~\ref{asm:beta} on real Qwen3-1.7B training rollouts. Both (a) the mean weighted divergence $\frac{1}{N}\sum_t w_t g_t(\beta)$ and (b) the mean weighted difficulty $\mathrm{load}(\beta)$ increase strictly and monotonically with the PI strength $\beta$.}
\label{fig:monotonicity}
\end{figure}

\begin{proposition}[The residual update targets the optimal PI strength]
\label{prop:direction}
Let $\beta$ be constrained to $[\beta_{\min}, 1]$ with $\beta_{\min} > 0$. Under
Assumption~\ref{asm:beta}, suppose there exists
$\beta^\star \in (\beta_{\min}, 1)$ with $\mathrm{load}(\beta^\star) = \varepsilon$.
Then $\beta^\star$ maximizes $\frac{1}{N}\sum_t w_t g_t(\beta)$ subject to
$\mathrm{load}(\beta) \le \varepsilon$, its multiplier satisfies
Eq.~\eqref{eq:beta_stationarity}, and the residual update
$\beta \leftarrow \beta + \eta_\beta(\varepsilon - \mathrm{load}(\beta))$
converges monotonically to $\beta^\star$ from either side.
\end{proposition}

\begin{proof}
Write $\ell(\beta) = \mathrm{load}(\beta)$ and $G(\beta) = \frac{1}{N}\sum_t w_t g_t(\beta)$,
both strictly increasing by Assumption~\ref{asm:beta}.

\emph{Step 1 (feasible set).} $\ell$ is injective, so $\beta^\star$ is the unique
root of $\ell(\beta) = \varepsilon$, and $\ell(\beta) \le \varepsilon \iff \beta \le \beta^\star$.
The feasible set is therefore $[\beta_{\min}, \beta^\star]$.

\emph{Step 2 (optimality).} $G$ is strictly increasing on this interval, so it is
maximized at $\beta = \beta^\star$, where the constraint is active. Stationarity at
this endpoint gives $G'(\beta^\star) = \lambda^\star \ell'(\beta^\star)$ with
$\lambda^\star > 0$, that is,
\[
\lambda^\star \;=\; \frac{\sum_t w_t\,\partial g_t/\partial\beta}
{\sum_t w_t\,\partial d_t/\partial\beta},
\]
recovering Eq.~\eqref{eq:beta_stationarity}.

\emph{Step 3 (update).} Since $\ell$ is increasing, $\varepsilon - \ell(\beta)$ has
the sign of $\beta^\star - \beta$. The update therefore increases $\beta$ when
$\beta < \beta^\star$, decreases it when $\beta > \beta^\star$, and is stationary
only at $\beta^\star$.
\end{proof}

Proposition~\ref{prop:direction} closes the gap between Eq.~\eqref{eq:beta_stationarity} and what can actually be computed during training: the primal residual drives $\beta$ to the very point that Eq.~\eqref{eq:beta_stationarity} characterizes, without ever forming a gradient in $\beta$. The residual therefore serves as a gradient-free surrogate for the optimal PI strength.

\subsection{USD: A Primal-Dual Online Algorithm}
\label{sec:algorithm}
The token weights, however, are not yet implementable. The optimal rule in 
Eq.~\eqref{eq:hard_weights} presumes the correct multiplier $\lambda$, which is a function of the current student and thus non-stationary over training; moreover, its $0/1$ indicator is not differentiable and cannot be incorporated into a gradient-based objective. Both $\lambda$ and $\mathbf{w}$ must therefore be estimated online, jointly with $\beta$.

We therefore propose \textbf{USD}, \textbf{U}nified on-policy \textbf{S}elf-\textbf{D}istillation, an online primal--dual algorithm that instantiates the
preceding optimization framework over the training stream. USD jointly adapts the token
weights $\mathbf{w}$, the dual variable $\lambda$, and the PI strength $\beta$, allowing the
supervision policy to continuously track the student's changing learning capacity.

\begin{table*}[h]
\centering
\caption{Math reasoning performance Avg@12 (\%) on Qwen3 backbones.  \textbf{Avg}: mean over the nine (model, benchmark) cells. \textbf{Bold}: best per column within each scale.}
\label{tab:main}
\setlength{\tabcolsep}{3pt}
\begin{tabular}{l ccc c ccc c ccc c c}
\toprule
& \multicolumn{3}{c}{\textbf{Qwen3-1.7B}} && \multicolumn{3}{c}{\textbf{Qwen3-4B}} && \multicolumn{3}{c}{\textbf{Qwen3-8B}} && \\
\cmidrule{2-4} \cmidrule{6-8} \cmidrule{10-12}
\textbf{Method} & \textbf{AIME24} & \textbf{AIME25} & \textbf{HMMT25} && \textbf{AIME24} & \textbf{AIME25} & \textbf{HMMT25} && \textbf{AIME24} & \textbf{AIME25} & \textbf{HMMT25} && \textbf{Avg} \\
\midrule
SFT           & 48.4          & 36.9          & 20.9          && 71.6          & 62.8          & 41.5          && 74.4          & 68.1          & 42.8          && 51.9 \\
GRPO          & 51.9          & 38.9          & 26.9          && 75.6          & 66.9          & 45.0          && 78.1          & 70.3          & \textbf{49.7} && 55.9 \\
OPSD          & 54.4          & 41.9          & 28.3          && 75.3          & 67.2          & 44.1          && 77.8          & 71.9          & 46.9          && 56.4 \\
TIP           & 55.8          & 41.9          & 27.2          && 75.8          & 69.7          & 45.0          && 76.9          & 72.8          & 49.2          && 57.1 \\
PAINT         & 56.4          & 42.8          & 29.2          && 75.6          & 66.7          & 44.2          && 78.3          & 72.5          & 47.2          && 57.0 \\
\textbf{Ours} & \textbf{58.9} & \textbf{43.6} & \textbf{29.4} && \textbf{76.7} & \textbf{70.0} & \textbf{46.7} && \textbf{79.7} & \textbf{73.3} & \textbf{49.7} && \textbf{58.7} \\
\bottomrule
\end{tabular}
\end{table*}

\paragraph{Update rules.}
The primal step on the student parameters requires differentiable weights, so we replace the hard threshold in Eq.~\ref{eq:hard_weights} by adding an entropy regularizer $\tau H(\mathbf{w})$ to the Lagrangian, yielding
\begin{equation}
w_t \;=\; \sigma\!\left(\frac{g_t(\beta) - \lambda\, d_t(\beta)}{\tau}\right),
\label{eq:soft_weights}
\end{equation}
where $\sigma$ is the sigmoid and $\tau > 0$ controls sharpness. Given these weights, we update the multiplier by projected subgradient ascent on the dual,
\begin{equation}
\lambda \;\leftarrow\; \max\!\bigl(0,\; \lambda + \eta_\lambda\, (\mathrm{load} - \varepsilon)\bigr),
\label{eq:dual_update}
\end{equation}
where $\mathrm{load} = \frac{1}{N}\sum_t w_t\, d_t(\beta)$ is the current mean weighted difficulty; when it exceeds the budget, $\lambda$ rises and tightens token selection through Eq.~\eqref{eq:soft_weights}. Finally, we step $\beta$ along the primal residual,
\begin{equation}
\beta \;\leftarrow\; \mathrm{clip}\!\bigl(\beta + \eta_\beta\,(\varepsilon - \mathrm{load}),\; \beta_{\min},\; 1\bigr),
\label{eq:beta_update}
\end{equation}
which by Proposition~\ref{prop:direction} drives $\beta$ to the point where $\mathrm{load}(\beta) = \varepsilon$. The dual update acts on the load in the same direction, since a larger $\lambda$ down-weights the highest-difficulty tokens. The two controls therefore act together on the load, raising it while the budget has slack and reducing it once the budget is exceeded, so the load converges to $\varepsilon$.

\paragraph{Training loss and algorithm.}
We normalize the training loss by total weight,
\begin{equation}
\mathcal{L}_{\text{train}} \;=\; \frac{\sum_t w_t\, D_{\text{KL}}(p_T^\beta\|p_S)_t}{\sum_t w_t},
\label{eq:loss}
\end{equation}
so token selection only shifts the relative importance of tokens, not the effective step size. Algorithm~\ref{alg:method} summarizes the full procedure. Beyond standard OPSD, the only per-batch overhead is one $O(T)$ pass to compute $\mathrm{load}$ and two scalar updates---no auxiliary networks or extra forward passes.

\begin{algorithm}[t]
\caption{The Algorithm of USD}
\label{alg:method}
\begin{algorithmic}[1]
\STATE \textbf{Input:} capacity $\varepsilon$, temperature $\tau$, step sizes $\eta_\lambda, \eta_\beta$
\STATE \textbf{Initialize:} $\lambda \leftarrow 0$,\; $\beta \leftarrow \beta_{\text{init}}$
\FOR{each training batch $(x, y^\star)$}
    \STATE Student rollout: $\hat{y} \sim p_S(\cdot \mid x)$
    \STATE Teacher scoring: $p_T^\beta(\cdot \mid x, \hat{y}_{<t})$ at PI strength $\beta$
    \STATE Per-token divergence: $g_t \leftarrow D_{\text{KL}}(p_T^\beta\|p_S)_t$
    \STATE Per-token difficulty: $d_t \leftarrow \left(\log \frac{p_T^\beta(a_t^\star)}{p_S(a_t^\star)}\right)_{+}$
    \STATE Soft weights: $w_t \leftarrow \sigma\bigl((g_t - \lambda\, d_t)/\tau\bigr)$
    \STATE Loss: $\mathcal{L} \leftarrow \sum_t w_t\, g_t\, /\, \sum_t w_t$
    \STATE Student update: gradient step on $\mathcal{L}$
    \STATE $\mathrm{load} \leftarrow \tfrac{1}{N}\sum_t w_t\, d_t$
    \STATE Dual ascent: $\lambda \leftarrow \max(0,\, \lambda + \eta_\lambda(\mathrm{load} - \varepsilon))$
    \STATE PI update: $\beta \leftarrow \mathrm{clip}(\beta + \eta_\beta(\varepsilon - \mathrm{load}),\, \beta_{\min},\, 1)$
\ENDFOR
\end{algorithmic}
\end{algorithm}

\section{Experiments}
\label{sec:experiments}

\subsection{Experimental Setup}

\paragraph{Models and Benchmarks.}
We use the Qwen3 family~\citep{yang2025qwen3} at three scales: 1.7B, 4B, and 8B, all in non-thinking mode. We utilize the mathematical reasoning subset of OpenThoughts~\citep{guha2025openthoughts} as training data, sampling up to 30K chain-of-thought augmented problem-solution pairs. Evaluation covers AIME 2024, AIME 2025, and HMMT 2025; we report Avg@12 on the best checkpoint per benchmark.

\paragraph{Baselines.}
We compare against SFT (imitates reference traces), GRPO~\citep{guo2025deepseek} (RL with verifiable rewards), vanilla OPSD~\citep{zhao2026self} (full-solution privileged teacher with uniform token weights), and two contemporaneous OPSD-style methods: TIP~\citep{xu2026tip}, which keeps the top-$\rho$ fraction of tokens by a Soft-OR entropy--divergence score, and PAINT~\citep{tan2026paint}, which combines overlap-adaptive suffix masking of the reference with sparse teacher--student logit interpolation. We reproduce TIP using the authors' open-source code with $\rho{=}0.2$, and PAINT following the paper's default settings. All baselines share our base models and training data; all methods train for 300 steps except GRPO, which we run for 500 steps and report at its peak checkpoint.

\paragraph{Implementation details.}
\textbf{(i) Training.} Following vanilla OPSD~\citep{zhao2026self}, the teacher is the frozen base model at the initial checkpoint; only the student's LoRA adapters~\citep{hu2022lora} (rank 64, $\alpha{=}128$) are updated with AdamW at learning rate $5{\times}10^{-6}$, effective batch size 32, gradient norm clip 0.1, for 300 steps.
\textbf{(ii) Hyperparameters.} PI strength $\beta \in [0,1]$ is instantiated as the fraction of reference tokens revealed to the teacher: at $\beta$, the teacher observes the first $\lfloor \beta L \rfloor$ characters of the reference (with $L$ its length in tokens), always retaining the final boxed answer. USD uses $\tau{=}0.1$, $\varepsilon{=}0.3$, $\eta_\lambda{=}0.1$, $\beta_{\text{init}}{=}0.8$, $\eta_\beta{=}0.03$, $\beta_{\min}{=}0.1$, fixed across all model scales, with a per-token KL clip of $0.05$ for numerical stabilization.
\textbf{(iii) Sampling.} Student rollouts are sampled at temperature $1.1$, top-$p$ $0.95$, top-$k$ $20$, and maximum length $1024$; at evaluation we sample 12 solutions per problem at temperature $1.0$ and report the best per benchmark over checkpoints saved every 50 steps.

\subsection{Main Results}
\begin{table*}[h]
\centering
\caption{Component ablation on Qwen3 backbones (Avg@12, \%). ``+ Token selection'': soft weights with fixed $\beta{=}1$; ``+ PI adaptation'': adaptive $\beta$ with uniform token weights. Within each scale, \textbf{Avg} is the mean over the three benchmarks and \textbf{bold} marks the best per column.}
\label{tab:ablation}
\setlength{\tabcolsep}{3pt}
\resizebox{\textwidth}{!}{%
\begin{tabular}{l cccc c cccc c cccc}
\toprule
& \multicolumn{4}{c}{\textbf{Qwen3-1.7B}} && \multicolumn{4}{c}{\textbf{Qwen3-4B}} && \multicolumn{4}{c}{\textbf{Qwen3-8B}} \\
\cmidrule{2-5} \cmidrule{7-10} \cmidrule{12-15}
\textbf{Variant} & \textbf{AIME24} & \textbf{AIME25} & \textbf{HMMT25} & \textbf{Avg}
                 && \textbf{AIME24} & \textbf{AIME25} & \textbf{HMMT25} & \textbf{Avg}
                 && \textbf{AIME24} & \textbf{AIME25} & \textbf{HMMT25} & \textbf{Avg} \\
\midrule
Vanilla OPSD         & 54.4          & 41.9          & 28.3          & 41.5          && 75.3          & 67.2          & 44.1          & 62.2          && 77.8          & 71.9          & 46.9          & 65.5          \\
\; + PI adaptation   & 58.3          & 40.8          & 28.1          & 42.4          && 75.3          & 69.4          & 44.2          & 63.0          && 78.1          & \textbf{74.2} & 48.6          & 67.0          \\
\; + Token selection & 58.6          & 42.5          & 28.9          & 43.3          && 75.0          & 67.8          & 45.0          & 62.6          && 78.6          & 73.6          & 47.8          & 66.7          \\
\; + Both (full)     & \textbf{58.9} & \textbf{43.6} & \textbf{29.4} & \textbf{44.0} && \textbf{76.7} & \textbf{70.0} & \textbf{46.7} & \textbf{64.5} && \textbf{79.7} & 73.3          & \textbf{49.7} & \textbf{67.6} \\
\bottomrule
\end{tabular}%
}
\end{table*}

Table~\ref{tab:main} compares all methods across three model scales and three benchmarks. Our method achieves the best Avg@12 on nearly all cells and raises the 9-cell overall mean from 56.4 (vanilla OPSD) to 58.7 (+2.3 points), with a consistent per-scale improvement of roughly two points at each of 1.7B, 4B, and 8B---the gain is not driven by any single benchmark or scale.

TIP and PAINT each improve over vanilla OPSD at most scales, confirming that token-level filtering and PI adaptation are individually beneficial. However, TIP operates at fixed $\beta{=}1$: it selects tokens under the strongest possible teacher, which means its selection boundary is tuned to a difficulty landscape that would shift under a softer teacher. PAINT adjusts PI strength but applies uniform token weighting, so the learning signal it optimizes includes many positions that contribute negligible divergence. Our method dominates both because the shared dual variable $\lambda$ simultaneously relaxes token selection (to include positions made feasible by a softer teacher) and adjusts PI strength (to exploit the capacity freed by filtering)---the complementary action that neither single-axis method can replicate.

Our method also outperforms GRPO at every scale (+4.7, +2.0, +1.5 on the per-scale average at 1.7B, 4B, and 8B), while using a far smaller sampling budget: a single $1024$-token rollout per problem over 300 steps, against 8 rollouts of up to 16k tokens over 500 steps for GRPO.

\begin{figure}[t]
\centering
\includegraphics[width=\columnwidth]{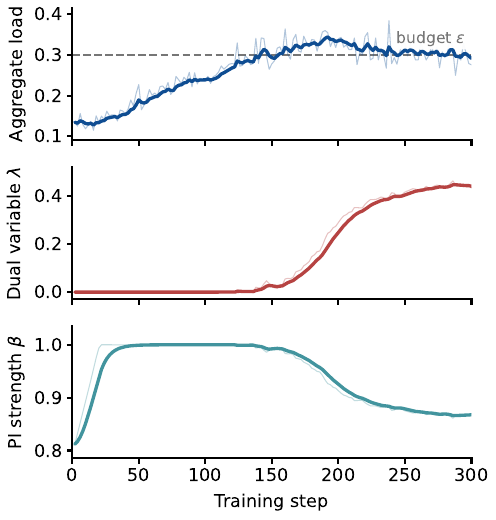}
\caption{Training dynamics of USD on Qwen3-1.7B.
Top: the aggregate load against the capacity budget $\varepsilon{=}0.3$ (dashed). Middle: the dual variable $\lambda$. Bottom: the PI strength $\beta$. Thin lines are per-step values, thick lines a running average. While the load is below budget, $\beta$ climbs to full exposure and $\lambda$ stays at zero; once the load reaches $\varepsilon$ , $\lambda$ activates and tightens token selection while $\beta$ eases back, and the load settles near $\varepsilon$.}
\label{fig:dynamics}
\end{figure}

\subsection{Ablation Study}

Table~\ref{tab:ablation} isolates each control axis on top of vanilla OPSD and yields three observations. \emph{First}, each axis is individually beneficial at every scale, confirming that both token selection and PI adaptation carry supervisory value in isolation. \emph{Second}, the joint variant attains the highest Avg@12 at every scale, with a gain of roughly two points over vanilla OPSD that is consistent across 1.7B, 4B, and 8B, whereas either axis alone yields a gain that varies with scale. The two axes are therefore complementary rather than 
redundant, since the shared capacity budget lets them substitute for one another without double-spending on the same tokens.\emph{Third}, no single axis dominates uniformly: token selection contributes more of the isolated gain at 1.7B, while PI adaptation edges ahead at 4B and 8B. Because the relative importance of the two axes shifts across scales, any fixed hand-picked weighting would be miscalibrated at one end of the range, whereas the coupled formulation---which derives both actions from a single capacity budget rather than a pre-set mixture---remains the top performer throughout.

\subsection{Empirical Analysis}

\paragraph{Co-evolution of $\lambda$ and $\beta$.} Figure~\ref{fig:dynamics} shows the training dynamics on Qwen3-1.7B. From its initial value, $\beta$ first rises: the teacher is still mild, the aggregate load sits below the budget, and the primal residual drives PI strength upward while $\lambda \approx 0$ leaves the weights nearly uniform. As $\beta$ increases, the teacher sharpens and per-token difficulty grows, pushing the load up to the budget. At that point $\lambda$ becomes active and starts filtering the highest-difficulty tokens, which in turn feeds back on $\beta$: the two controls thereafter move together, with $\lambda$ tightening selection while $\beta$ eases the teacher, keeping the aggregate load near the budget. Both eventually level off once the load stabilizes around $\varepsilon$.

\begin{table}[t]
\centering
\caption{Sensitivity to capacity budget $\varepsilon$ on Qwen3-1.7B (best Avg@12 across checkpoints, averaged over three benchmarks).}
\label{tab:sensitivity}
\begin{tabular}{lcccccc}
\toprule
$\varepsilon$ & 0.15 & 0.20 & \textbf{0.30} & 0.40 & 0.50 & 0.80 \\
\midrule
Avg & 42.0 & 41.5 & \textbf{44.0} & 40.8 & 41.5 & 41.0 \\
\bottomrule
\end{tabular}
\end{table}

\paragraph{Sensitivity to $\varepsilon$.}
Table~\ref{tab:sensitivity} reports the average accuracy as the capacity budget $\varepsilon$ varies. Performance peaks at $\varepsilon{=}0.3$ and degrades in both directions: an overly tight budget over-filters and starves the student of learnable signal, while an overly loose budget leaves the constraint rarely active, so the method reverts to near-uniform weighting. The peak at an intermediate value indicates a genuine trade-off: the gain comes from matching the budget to what the student can absorb, not from filtering as much as possible.

\section{Conclusion}
On-policy self-distillation has become an effective post-training paradigm for LLM reasoning, using only the model's own rollouts and a same-model privileged teacher. Existing extensions improve vanilla OPSD either by token re-weighting or by adapting privileged information, but optimize the two mechanisms independently. We show that they are coupled through a shared learning-capacity budget, yielding a unified optimization whose Lagrangian reduces both decisions to a single dual variable. Based on this formulation, we propose USD, a lightweight online algorithm that prices learning difficulty once and lets that price govern both axes. It consistently improves over vanilla OPSD, and ablations show that the two axes contribute complementary gains. USD establishes a capacity-centred view of OPSD, in which token weighting and privileged information are two expressions of one underlying budget. We hope this view provides a useful basis for future work on how much supervision a student should receive, and when.

\section{Acknowledgement}
We also sincerely thank Yi Liu and the Workbuddy \& CodeBuddy Team in Tencent (https://www.codebuddy.ai/) for their valuable assistance throughout this work.

\bibliography{main}

\end{document}